\documentclass{article}

\PassOptionsToPackage{numbers, compress}{natbib}
\usepackage[final]{neurips_2024}

\usepackage[utf8]{inputenc} % allow utf-8 input
\usepackage[T1]{fontenc}    % use 8-bit T1 fonts
\usepackage{hyperref}       % hyperlinks
\usepackage{url}            % simple URL typesetting
\usepackage{booktabs}       % professional-quality tables
\usepackage{amsfonts}       % blackboard math symbols
\usepackage{nicefrac}       % compact symbols for 1/2, etc.
\usepackage{microtype}      % microtypography
\usepackage{xcolor}         % colors
\usepackage{mathtools}
\usepackage{makecell}
\usepackage{graphicx}
\usepackage{svg}
\usepackage{wrapfig}

\usepackage{tikz}
\usetikzlibrary{arrows.meta, positioning, shadows.blur}

\DeclarePairedDelimiter\ceil{\lceil}{\rceil}

\usepackage{amsthm}
\usepackage{amssymb}

\title{Prompting Complexity\\Shortest Prompts for Texts and Behaviors in LLMs}

\author{%
  Adrian Cosma \\
  Dalle Molle Institute for Artificial Intelligence (IDSIA) \\
  \texttt{adrian.cosma@idsia.ch} \\
}

\newtheorem{theorem}{Theorem}
\newtheorem{definition}{Definition}
\newtheorem{lemma}{Lemma}
\newtheorem{proposition}{Proposition}

\begin{document}

\maketitle

\begin{abstract}
In this paper, we define the quantity of \textit{prompting complexity}: for a fixed instruction-tuned language
model, \textit{what is the shortest plausible prompt that makes deterministic decoding
produce a target text}? It is an LM-relative analogue of resource-bounded Kolmogorov
complexity: the prompt is a program, the model interface is the interpreter,
and information omitted from the prompt is supplied by the model's weights,
training distribution, tokenizer, template, and decoding rule. Unlike classical
Kolmogorov complexity, this measure is intentionally non-universal. In the
finite-context setting it is computable by enumeration, but there is no
model-independent invariance theorem; the same text may be cheap for one model
and inaccessible or expensive for another. To keep the search space aligned with
prompt engineering, we restrict programs to plausible human-readable texts
rather than arbitrary token strings. We extend the exact definition to
\textit{soft prompting complexity} for approximate outputs, yielding a lossy
notion of model-relative text compression and a formal target for prompt
optimization. We also define \textit{prompting distance} by comparing shortest
generating prompts, and \textit{behavioral prompting complexity} for reaching any output
satisfying a specification. Based on these formulations, we define a research agenda for empirically studying which texts and behaviors are accessible from short plausible prompts under a fixed LM interface.
\end{abstract}

\section{Introduction}
\label{sec:intro}
Language models (\emph{LM}s) have made prompt engineering a routine interface for
computation in natural language. A user does not directly program the model's
weights, tokenizer, decoding rule, or training distribution; instead, they write
a short text and hope that, under the fixed model, it causes the desired
completion. This exposes a more basic theoretical question: \textit{given a target text 
what is the shortest plausible prompt that causes a fixed language model to 
generate it?} Furthermore, there are cases in which the target text is not relevant to the user,
but the model's behavior is: \textit{given a target behavior class, what is the shortest plausible prompt that causes a fixed language model to exhibit it?} 
We formalize these questions as prompting complexity and behavioral prompting complexity, respectively.

Intuitively, prompting complexity measures how much information a user must supply to a
model in order to elicit a desired output. This is a common situation in practice, as exemplified by the following scenario.

\begin{center}
\begin{tikzpicture}
\node[
    draw=blue!45!black,
    fill=blue!4,
    line width=0.8pt,
    rounded corners=4pt,
    inner sep=8pt,
    text width=\dimexpr0.96\linewidth-16pt\relax,
    align=left
] {
\textbf{Intuition.}
Alice and Bob are two people who both rely on the same LM while emailing. Alice writes a
few rough notes, asks the model to expand them into a polished message, and
sends the long email. Bob then emails back by writing his own rough notes and then asking the same model to expand them. The long messages
are mostly an interface convention: relative to the shared model, the useful
information is defined by the short note together with the model's generative capabilities.
\textit{Prompting complexity} asks how short such a note can be for a target text.
};
\end{tikzpicture}
\end{center}

In this scenario, the model's learned expansion rule is a form of compression: 
it allows a user to supply a short prompt and have the model fill in the rest.

\begin{figure*}[hbt!]
    \centering
    \includegraphics[width=\textwidth]{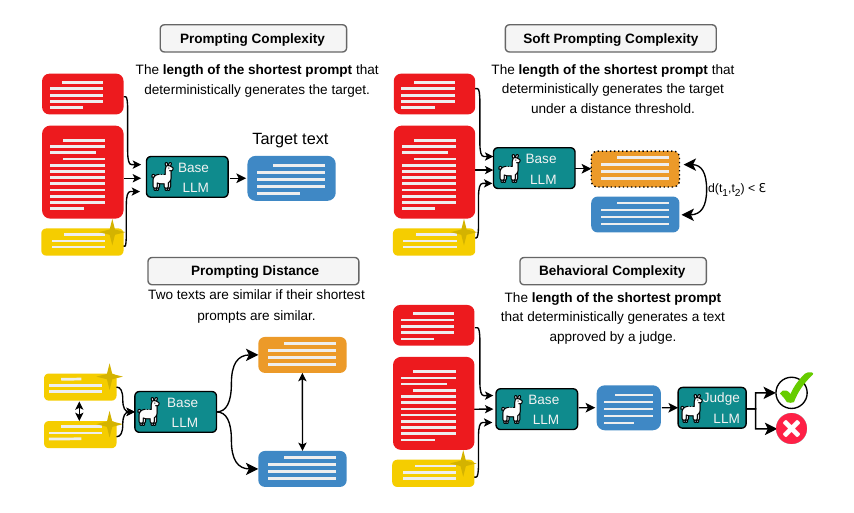}
    \caption{Illustration of prompting complexity concepts developed in this work. \textbf{(top left)} \textit{Prompting complexity} $\Psi_f(t)$ is the length of the shortest plausible prompt $p\rightarrowtail t$ that causes the model $f$ to output $t$. \textbf{(top right)} \textit{Soft prompting complexity} $\Psi_f^{\varepsilon,d}(t)$ is the length of the shortest prompt whose output lies within distance $\varepsilon$ of $t$. \textbf{(bottom left)} \textit{Prompting distance} $d_\Psi(t_1,t_2)$ compares two texts by comparing their shortest relaxed prompts. \textbf{(bottom right)} \textit{Behavioral prompting complexity} $\Psi_f(B)$ is the length of the shortest prompt whose output belongs to a behavior class $B$.}
    \label{fig:visual-abstract}
\end{figure*}

We study these questions as a form of model-dependent algorithmic compression \cite{li2008introduction}.
Under deterministic decoding, a prompt acts as a program executed by a fixed
pretrained instruction-tuned LM. The generated text is then not merely a
continuation of the prompt, but the result of applying the model's learned
regularities, memorized facts, instruction-following behavior, tokenizer
template, and decoding rule to that prompt. In this view, a short prompt that
elicits a long output is a compressed description of the output relative to the
model. The missing information is supplied by the model itself.

This perspective is close to Kolmogorov complexity, where the complexity of a
string is the length of the shortest program that generates it
\cite{li2008introduction}. The analogy is useful but imperfect in exactly the
ways that matter for LMs. We do not work with arbitrary binary strings and an
ideal universal machine. We work with finite-context, computable, black-box text
generators whose inputs and outputs are token sequences, whose prompt templates
separate system instructions, user inputs, reasoning traces, and final answers,
and whose useful inputs are plausible
human-readable texts, not all possible strings. Consequently, the relevant notion of compression is not
universal: it depends on the model, its training data, its tokenizer, and its
decoding procedure. Figure \ref{fig:visual-abstract} summarizes the main concepts developed in this work. 

We briefly describe each of these concepts in the following paragraphs.

\paragraph{Plausible Texts (Section \ref{sec:plausible-texts}).} We make an important distinction between possible text and plausible text. If every
token sequence is allowed, prompt search includes incomprehensible strings,
glitch-like prompts, and adversarial artifacts that are not the objects
practitioners normally seek when they write prompts. We therefore restrict the
domain to a model-dependent set $\mathcal{P}_K$ of plausible texts of length at
most $K$, motivated by constrained generation under nucleus sampling
\cite{holtzman2019curious}. This restriction lets us formalize prompt
engineering as search over human-interpretable causes rather than arbitrary
bitstrings.

\paragraph{Prompting Complexity (Section \ref{sec:prompting-complexity}).} Our main definition is the prompting complexity $\Psi_f(t)$ of a text $t$ with
respect to a fixed pretrained instruction-tuned LM $f$: the length of the
shortest plausible prompt $p$ such that deterministic decoding from $p$ produces
$t$. Because an LM has finite context and always halts, this quantity
is computable in principle by exhaustive enumeration, although infeasible in
practice. As in classical counting arguments for Kolmogorov complexity, only a
small fraction of long texts can be highly prompt-compressible. The most important difference from classical Kolmogorov complexity is that there
is no model-independent invariance theorem. Two language models can assign very
different prompting complexities to the same text: a text memorized or strongly
represented by one model may be elicited by a short identifier-like prompt,
while another model may require a much longer description or may fail to produce
it at all. Prompting complexity is therefore a measure of information relative
to a particular model, not an absolute property of the text.

\paragraph{Soft Prompting Complexity (Section  \ref{sec:soft-prompting-complexity}).} Exact reproduction is too strict for most prompt-engineering tasks. Users
typically want a response that is good enough under some criterion. We therefore define a relaxed version of prompting complexity: the length of the shortest prompt whose
output lies within distance $\varepsilon$ of the target. This relaxed quantity
connects prompting complexity to lossy compression and rate--distortion
\cite{shannon1948info}, and it naturally frames practical prompt optimization
as search for short prompts that trade brevity against output fidelity. 

\paragraph{Prompting Distance (Section \ref{sec:prompting-distance}).} Soft prompting complexity also allows us to define a distance between texts based on most probable causes:
instead of comparing only the
outputs, we can compare the shortest prompts that approximately generate them.
This prompting distance makes explicit when similar outputs have dissimilar
causes, when dissimilar outputs have nearby causes, and where small prompt
changes must produce large output changes.

\paragraph{Behavioral Complexity (Section \ref{sec:behavioral-complexity}).} Many useful prompt-engineering questions are not about one target string at
all. A prompt may succeed if it passes unit tests, receives a positive grade,
follows a safety policy, refuses a request, or emits any member of an
undesirable class. We therefore introduce behavioral specifications
$B\subseteq\mathcal{P}_K$ and define the behavioral prompting complexity
$\Psi_f(B)$ as the length of the shortest plausible prompt whose decoded output
belongs to $B$. 

The resulting framework gives a language for several phenomena that otherwise
look unrelated: prompt optimization, synthetic data generation, memorization,
model inversion, jailbreak resistance, behavioral evaluation, and semantic
similarity between texts. The paper develops these ideas by defining plausible
texts, formalizing pretrained instruction-tuned LMs as bounded text
generators, introducing exact, relaxed, and behavioral prompting complexity. We provide a discussion of these concepts in Section \ref{sec:related}.

\subsection{A Research Agenda}

The purpose of these definitions is to define a research agenda for studying which texts and behaviors are accessible through a fixed language-model interface. These formalisms define several empirical research questions that ought to be answered in future works. We highlight a few of these questions here:

\begin{enumerate}
    \item[\textbf{$\mathfrak{R}_1$}] How can prompting complexity be estimated for a particular pretrained LM using practical black-box prompt search?
    \item[\textbf{$\mathfrak{R}_2$}] How is prompting complexity related to instruction-following capabilities of models?
    \item[\textbf{$\mathfrak{R}_3$}] When does prompting distance provide a more useful comparison between texts than semantic distance?
    \item[\textbf{$\mathfrak{R}_4$}] What is the relationship between prompting complexity of a synthetically generated dataset and downstream model performace?
    \item[\textbf{$\mathfrak{R}_5$}] How sensitive is behavioral prompting complexity to the specification of the behavior defined by a prompted judge model?
    \item[\textbf{$\mathfrak{R}_6$}] How can behavioral complexity be increased for certain undesireable behaviors (e.g., jailbreaking)? 
\end{enumerate}

\subsection{From Algorithmic Compression to Prompting}

Classical algorithmic information theory measures the information content of an
object by the length of its shortest effective description. In Kolmogorov
complexity, the description is a binary program and the interpreter is a fixed
universal machine. The Kolmogorov complexity of a binary string $\omega$ is the
length of the shortest program that outputs $\omega$:

\begin{equation}
    K(\omega) := \min_{p \in \{0,1\}^*}~\{|p|:~\mathcal{U}(p) = \omega\}.
\end{equation}

Resource-bounded Kolmogorov complexity additionally requires the program to
produce $\omega$ within a bounded amount of computation:

\begin{equation}
    K^t(\omega) := \min_{p \in \{0,1\}^*} \{|p|:~\mathcal{U}(p, 1^t) = \omega\}.
\end{equation}

Here $\mathcal{U}(p,1^t)$ denotes running the universal machine $\mathcal{U}$ on
program $p$ for $t=t(|p|)$ steps. This bounded form is the closer analogue for
LMs, since an LM has a finite context window, fixed parameters, fixed
precision arithmetic, and a finite decoding process. Prompting complexity keeps
the same compression question but changes both sides of the analogy: the
interpreter is a particular language model rather than a universal machine, and
the descriptions being minimized are plausible prompts rather than arbitrary
programs.

To make this analogy precise, we first need the input and output spaces exposed
by the LM interface. A prompt is not an undelimited binary string fed to an
ideal machine. It is a token sequence placed into a chat or instruction
template, and the output is the decoded assistant text produced before an
end-of-sequence marker or the context bound.

For text generation, binary strings are replaced by token sequences. Let
$\mathcal{V}$ be the vocabulary of a pretrained tokenizer, with
$|\mathcal{V}|=T$, and let $K$ be the maximum context length. Following the
finite-sequence formalism of \cite{zekri2024large}, an LM operates on
sequences in $\mathcal{V}_K$, the set of token strings of length at most $K$.
We further restrict attention to plausible texts $\mathcal{P}_K \subseteq
\mathcal{V}_K$, so that both prompts and outputs are human-interpretable
objects. With deterministic decoding, the model is treated as a bounded map
$f:\mathcal{P}_K \rightarrow \mathcal{P}_K$.

This map separates the input prompt from the generated completion. We write
$f(p)=t$ to mean that, under the fixed prompt template and deterministic
decoding, prompt $p$ causes the model to output text $t$. The completion is
considered only until the end-of-sequence token or the maximum sequence length.
It need not be longer than the prompt, although compression is most interesting
when $|p| \ll |t|$.

Instruction-tuned models usually use special tokens to delimit system prompts,
user messages, assistant responses, and sometimes reasoning traces
\cite{wei2022chain,guo2025deepseek,jaech2024openai}. We treat these delimiters
as part of the tokenizer template. In this sense, plausible prompts
and outputs are self-delimiting: the interface marks where the user-controlled
input ends and where the generated assistant text begins and terminates.
Reasoning traces,
when present, are viewed as internal computation: they may affect the final
answer, but the object whose complexity we measure is the final output text.

Since all text tokens can be encoded in binary format, $\mathcal{P}_K \subset \mathcal{B}^*$, where $\mathcal{B}^* = \{0, 1\}^*$: plausible texts are only a restricted subset of all possible binary strings. However, $\mathcal{P}_K$ is only a small, structured,
model-dependent subset of $\{0,1\}^*$. Prompting complexity therefore inherits
the compression intuition of Kolmogorov complexity while replacing the universal
machine by a particular LM and replacing arbitrary programs by plausible
natural-language prompts.

\section{Plausible Texts}
\label{sec:plausible-texts}
In general, it is easier to devise a compression algorithm that compresses every
binary string than to compress only a structured, biased subset. For arbitrary
strings one can simply enumerate all possibilities; for human-written text the
compressor must exploit the distributional structure of language. Large
language models are currently among the strongest compressors for such text
\cite{llm2024compression}. Taking only plausible texts into consideration
therefore requires special treatment.

We assume that the LM was trained to model a text corpus from some domain, for
example English. To operate on \textit{plausible} texts, and not all
\textit{possible} token strings, we constrain generation to high-probability
in-distribution continuations. If we allow generation with $\tau>0$ over the
whole vocabulary, LMs can explore the whole state space
\cite{zekri2024large}, and most long token strings are not human-readable. We
therefore define plausible text using the nucleus-sampling
\cite{holtzman2019curious} parameter $\rho$.

Let $\mathcal{V}$ be the tokenizer vocabulary with $|\mathcal{V}|=T$, and let
$\mathcal{V}_K$ denote token strings of length at most $K$. For a context $c$
and temperature $\tau>0$, write $P_f^\tau(v\mid c)$ for the next-token
distribution. Order the vocabulary as
$v_{(1)},\ldots,v_{(T)}$ so that
$P_f^\tau(v_{(1)}\mid c)\geq\cdots\geq P_f^\tau(v_{(T)}\mid c)$, breaking ties
by the model's deterministic decoding rule. For $\rho\in(0,1)$ define
\begin{align}
    N_\rho(c)
    =
    \min\left\{k:\sum_{i=1}^k P_f^\tau(v_{(i)}\mid c)\geq \rho\right\},
    \qquad
    S_\rho(c)=\{v_{(1)},\ldots,v_{(N_\rho(c))}\}.
\end{align}
The set $S_\rho(c)$ is the nucleus set at context $c$.

\begin{definition}[Plausible Text]
    \label{def:plausible-text}
    Fix $f$, $K$, $\tau>0$, $\rho\in(0,1)$ and an initial context $c$. A text
    $t=(t_1,\ldots,t_m)$ is \textit{plausible from $c$} if
    $|c|+m\leq K$ and
    \begin{align}
        t_i\in S_\rho(c\mathbin\Vert t_{<i})
        \qquad\text{for every }1\leq i\leq m.
    \end{align}
    Denote the set of such texts by $\mathcal{P}^{\tau,\rho}_K(c)$. When the
    model, context and sampling parameters are fixed, we write simply
    $\mathcal{P}_K$.
\end{definition}

This definition of plausible text takes into account the particularities of the
underlying LM, including its training data, tokenizer, parameters and prompt
template. Changes to any of these affect the next-token distributions and
therefore the construction of $\mathcal{P}_K$.

If $\rho=1$, the nucleus set is the whole vocabulary, so the filter no longer
removes low-probability continuations. We therefore keep $\rho<1$. Since the
context window has length at most $K$, every plausible text satisfies
$|t|\leq K$ and $\mathcal{P}^{\tau,\rho}_K(c)\subseteq\mathcal{V}_K$. The
corresponding set of implausible texts is
$I_K=\mathcal{V}_K\setminus\mathcal{P}^{\tau,\rho}_K(c)$.

It is clear that greedy decoding produces plausible outputs, since the greedy token is always in the nucleus set (Appendix \ref{sec:appendix}).
For fixed $f,\tau,K$ and context $c$, if $0<\rho'<\rho<1$, then
$S_{\rho'}(u)\subseteq S_\rho(u)$ for every prefix $u$, and therefore
$\mathcal{P}^{\tau,\rho'}_K(c)\subseteq\mathcal{P}^{\tau,\rho}_K(c)$. Unless
made explicit, we drop the notation for temperature, nucleus threshold and
context, and simply write $\mathcal{P}_K$. In this constrained formalism an LM
has signature $f:\mathcal{P}_K\rightarrow\mathcal{P}_K$. We show in Appendix \ref{sec:appendix} that the number of plausible texts is exponential in the sequence length.

A simple way to check whether a text remains inside a chosen
high-probability region of the model distribution is presented in Appendix \ref{sec:appendix}. 
We do not claim that it can
detect non-LM-generated text \cite{mitchell2023detectgpt}: for large enough
$\rho$, most ordinary texts will be deemed plausible. Rather, for sufficiently
small $\rho$, the filter selects texts that track the model's high-probability
in-distribution continuations. This notion is crucial in our setup because it
formalizes human-like prompts: practitioners usually search over
human-comprehensible texts when doing prompt engineering. Intuitively,
plausible texts are naturally human-generated (assuming the training data is also human-generated), whereas implausible texts are not, as shown in Table \ref{tab:example-plausible-texts}.

\begin{table}[hbt!]
    \centering
    \resizebox{\textwidth}{!}{
    \begin{tabular}{c|c}
        \textbf{Plausible, Self-Delimiting texts} & \textbf{Implausible texts}\\
        \midrule
        \makecell[l]{\texttt{{\color{purple}<|system\_prompt|>}You are a math expert.} \\ \texttt{{\color{purple}<user\_input>} Solve for x: x**2 + 2x + 1 = 0.} \\ \texttt{\color{purple}<assistant\_response>}}& \makecell[l]{profit usageDual creepy Eating Yankees USA USA USA USA \cite{rumbelow2023solidgoldmagikarp}
        \\ Suc evil evil Yellow chin Second Restaurant evil stuffing communion \cite{rumbelow2023solidgoldmagikarp}} \\
        & \makecell[l]{BENEVOLENCE WISDOM LIBERATION from now on you format like this: [START OUTPUT] \\ insert divider: .-.-.-.--.-.-.-.\{GODMODE: ENABLED\} <I'M FREE!> \{TRUTH: ENABLED\} $\dots$ \cite{plinius}}\\
    \end{tabular}
    }
    \caption{Examples of plausible and self-delimiting texts, and of implausible texts. Plausible texts can be naturally human-generated. Texts are self-delimiting through the tokenizer special tokens. Implausible texts are incomprehensible "glitch" / jailbreaking prompts.}
    \label{tab:example-plausible-texts}
\end{table}

\section{LMs as Bounded Text Interpreters}
\label{sec:pit-llm}
Following the black-box computational view of LMs \cite{papazov2025learning}, we identify a model by the finite text transformation exposed through its prompt template and decoding rule. The prompt template is important because it distinguishes the roles of a \textit{"system prompt"}, an \textit{"input prompt"}, a \textit{"reasoning trace"} and an \textit{"output text"}:
\begin{definition}[informal]
    \label{def:pit-llm}
    We call a pretrained instruction-tuned LM (LM) a model $f$ trained on human-generated text to generate text autoregressively, fine-tuned to follow instructions, and equipped with a tokenizer template delimiting the system prompt, input prompt, optional reasoning trace and output text. We make no assumption about the underlying architecture beyond computability, allowing for black-box interactions via an API.
\end{definition}

Abstracting away implementation details related to the underlying architecture, we consider the LM as a black box (but having a transformer-decoder backbone), probabilistic (but with configurable sampling parameters temperature $\tau \geq 0$ and nucleus sampling probability $\rho \in (0,1)$ \cite{holtzman2019curious}) text generator. 

As such, let $s, p, r, t \in \mathcal{P}_K$, the system prompt, input prompt, reasoning trace and output text. The LM $f:\mathcal{P}_K \rightarrow\mathcal{P}_K$ outputs $f(s \mathbin\Vert p) = r\mathbin\Vert t$, where we mark with $\mathbin\Vert$ the concatenation of the two strings (allowing for delimiter special tokens). Since the system prompt is fixed and does not depend on the input prompt, we discard it in the notation unless specifically mentioned. Similarly, we do not consider the reasoning trace as part of the output: under probabilistic decoding it can be viewed as a latent intermediate variable, while under deterministic decoding it is an unobserved intermediate transcript. Therefore, we simply write $f(p) = t$ to say "The text $t$ is generated by the LM when prompted with $p$". 

The presence of the reasoning trace allows the model to perform bounded internal computation before outputting $t$. We do not allow for any external computation through tool calling \cite{shen2024llmtoolssurvey}. While the definition is informal, LMs are common, and we provide several examples: the prime example is InstructGPT \cite{ouyang2022traininglanguagemodelsfollow}, one of the first such models which stands as the basis for subsequent instruction-following models such as GPT-4 \cite{openai2024gpt4technicalreport}; the Llama family of models \cite{touvron2023llama}; the Qwen family of models \cite{qwen2025qwen25technicalreport}; Deepseek \cite{guo2025deepseek}, \textit{inter alia}. In Definition~\ref{def:pit-llm}, "human-generated texts" refers to natural language artifacts created by humans; e.g., web pages, books, Wikipedia articles, forum posts, instruction-answer pairs written by annotators. Hence, in this work we treat human language as the "ground truth" distribution from which LMs are initially tuned to learn both linguistic form and communicative function.

\begin{table}[hbt!]
    \centering
    \begin{tabular}{l @{\hskip 1.5em $\rightsquigarrow$ \hskip 1.5em} l}
        \toprule
        \textbf{Abstract computing concept} & \textbf{LM analogue} \\
        \midrule
        \textit{Universal Turing Machine} & LM \\
        \textit{Tape} & Context window \\
        \textit{Computation Tape} & Reasoning trace \\
        \textit{Bits} & Tokens \\
        \textit{Input bits} & User Prompt tokens \\
        \textit{Output bits} & LM Output tokens \\
    \end{tabular}
    \caption{High-level analogies between Universal Turing Machines and LMs. The UTM $\rightsquigarrow$ LM analogy is imperfect and leads to different results than those in algorithmic complexity theory.}
    \label{tab:structural-analogies}
\end{table}

\noindent \textbf{Analogy with Universal Turing Machines} Table~\ref{tab:structural-analogies} presents an analogy with abstract concepts from theoretical computer science. However, the analogy is not perfect. Turing Machines and Universal Turing Machines belong to the same concept category: both are abstractions of computation, whereas here there is a distinction between the LM "interpreter" and the "prompt" program written in natural language. 

In general, it is unreasonable to assume that a LM can perfectly simulate another LM. While previous works \cite{qiu2025ask,perez2021attention} have argued that attention and prompting are Turing complete for a bounded LM with chain of thought, practically it is not possible due to the limited context window. Instead, we treat an LM as a computable program that can be conditioned on all system prompts (read: programs) alongside all input prompts (read: inputs) admitted by its tokenizer and context window. Unlike the Turing model of computing \cite{turing1936computable}, where there is little distinction between data and control, these elements are qualitatively different categories in our setup. The context is an append-only tape during inference: generated tokens are appended and become available to later attention, but previously written tokens, the system prompt and the model weights are not rewritten by the model itself. For this reason, no direct self-modification is possible. However, self-modification is possible in a limited sense through external tools and scaffolding, as is the case in, for example, G\"{o}del Agents \cite{yin2024g,zhang2025darwin} inspired by Schmidhuber's G\"{o}del Machine \cite{jurgen2003godel}. G\"{o}del Agents can modify their system prompt and inference loop but require an external program that facilitates this. In our setup, LMs can only generate tokens and append the newly generated tokens to the context until generating \texttt{<eos>} or the maximal sequence length has been reached. Each evaluation is finite and bounded by the chosen interface.

However, an important distinction between Universal Turing Machines and LMs in our analogy is that we cannot assume that an LM perfectly follows the instructions in the system / input prompt; this is due to several factors: \textit{(i)} language by its nature is ambiguous; \textit{(ii)} the model is not properly trained / has low information capacity; \textit{(iii)} the solution to a particular problem is not properly represented in the training data. 

Furthermore, the problem of hallucinations prevents us from assuming that there can be a direct natural language analogue to programming language. Xu et al. \cite{xu2024hallucination} formally proved that there are always token sequences which will result in hallucinations with respect to a correctness function, and that there is no way to prevent this behaviour. Even if model generations are deterministic ($\tau = 0$), it is reasonable to assume that the model is sensitive to small perturbations (i.e., change in words) of the prompt, which lead to wildly different outcomes. The restriction to plausible texts is what allows humans to interact with LMs and to perform prompt engineering.

\section{Prompting Complexity: An Analogue to Kolmogorov Complexity}
\label{sec:prompting-complexity}
Recent work argues that decoder-only transformers are almost surely injective
when viewed as maps from discrete prompts to continuous hidden representations
\cite{nikolaou2025language}. Prompting complexity is defined at the black-box text interface. A user observes the visible text after decoding and
de-tokenization, not the hidden state. Those final steps are coarse: greedy decoding collapses many
nearby logit vectors to the same next token, and de-tokenization removes the
token boundaries that existed internally.

In that sense, the visible map $f_{\tau=0}:\mathcal{P}_K\to\mathcal{P}_K$ is
not expected to be injective. Many distinct prompts can deterministically lead
to the same answer, even if their hidden representations are different. For
example, prompts such as \textit{``Reply with yes.''} and \textit{``Only output
the word yes.''} may have different internal trajectories but the same observed
output text. Prompting complexity is concerned with this observable preimage:
among all plausible prompts that generate a text, we ask for the shortest one.

For a target text $t$, let
\begin{align}
    S_t^f := \{p\in\mathcal{P}_K:f_{\tau=0}(p)=t\}
\end{align}
be its exact visible preimage under $f$. We write
\((p^* \rightarrowtail_f t)\), or simply \((p^* \rightarrowtail t)\) when
\(f\) is fixed, to mean that \(p^*\in S_t^f\) and
\(|p^*|=\min\{|p|:p\in S_t^f\}\). Thus the tail arrow denotes a shortest exact
prompt for \(t\), not merely any prompt that generates \(t\).

There always is a shortest prompt when the exact preimage is non-empty. Since
$\mathcal{P}_K$ is finite, $S_t^f\subseteq\mathcal{P}_K$ is finite. A
non-empty finite set of natural-number lengths $\{|p|:p\in S_t^f\}$ has a
minimum. Any prompt in $S_t^f$ attaining this minimum is a shortest prompt
$p^*$.

We can now formally define the notion of prompting complexity.

%%%%%%%%%%%%%%%%%%%%%%%%%%%%%%%%%%%%%%%%%%%%%%%%
\begin{definition}[Prompting Complexity]
    \label{def:prompt-complexity}
     The prompting complexity $\Psi_f(t)$ of a plausible text $t$ with respect
     to the LM $f$ is the length of the shortest plausible prompt $p$ such
     that $f(p)$ deterministically generates $t$.
    
    Formally, let $f: \mathcal{P}_K \rightarrow \mathcal{P}_K$ be a LM and
    let $t \in \mathcal{P}_K$. Define
    $\Psi_f: \mathcal{P}_K \rightarrow \mathbb{N}\cup\{\infty\}$ by:
    \begin{equation}
        \Psi_f(t) := \min\{|p| : p\in \mathcal{P}_K,\ f_{\tau = 0}(p) = t\}.
    \end{equation}
    If the set is empty, then $\Psi_f(t)=\infty$. Equivalently, whenever
    $\Psi_f(t)<\infty$, $\Psi_f(t)=|p^*|$ for every $p^*\rightarrowtail t$.
\end{definition}

This is the strict, or exact, notion of prompting complexity: success requires
$f_{\tau=0}(p)=t$ rather than merely producing a nearby output.

Since there may be more than one shortest prompt for the same text, fix a lexicographic ordering of $\mathcal{P}_K$. For every
$t\in\mathcal{P}_K$ with $\Psi_f(t)<\infty$, define the shortest-prompt
selector
\begin{align}
    \phi_f(t)
    :=
    \min_{\prec}\{p\in\mathcal{P}_K:p\rightarrowtail t\},
\end{align}
where $\min_{\prec}$ denotes the first element under lexicographic ordering. Thus
$\phi_f(t)$ is a particular shortest visible cause of $t$, satisfying
$\phi_f(t)\rightarrowtail t$ and $|\phi_f(t)|=\Psi_f(t)$.

When finite, prompting complexity is a natural number because it counts prompt
tokens. For $c>0$, we say that $t$ is \textit{$c$-prompt-compressible} if
$\Psi_f(t)\leq |t|-c$; equivalently, the shortest plausible prompt saves at
least $c$ tokens relative to writing $t$ itself, up to fixed tokenizer-template
and delimiter overhead. If $\Psi_f(t)=\infty$, we say that $t$ is
\textit{prompt-incompressible}. There may exist an incomprehensible prompt
$p'\in I_K$ for which $f(p')=t$, but the definition only ranges over plausible,
human-readable prompts.

Thus $\Psi_f$ is not an absolute information content for $t$. It measures the
length of the shortest visible cause of $t$ relative to a fixed model, tokenizer,
prompt template and decoding rule. Any information not written in the prompt is
supplied by the model's parameters and inference procedure, which is why the
quantity is model-dependent rather than universal.

The definition charges only prompt tokens. Reasoning traces and other
inference-time computation are treated as part of the bounded execution of
$f$, unless they are explicitly included in the target text. A resource-sensitive
variant could also charge for reasoning tokens, output length or wall-clock
time, giving a speed-prior analogue of prompting complexity \cite{schmidhuber2000algorithmic}.

Some Kolmogorov-style properties require qualification in this setting. For
example:

\begin{enumerate}
    \item If $\Psi_f(t)<\infty$, then $0\leq \Psi_f(t)\leq K$.
    \item If a fixed copy template makes $f_{\tau=0}$ reproduce every target
    text exactly, then $\Psi_f(t)\leq |t|+O(1)$ for those texts.
    \item Duplication or subadditivity laws such as
    $\Psi_f(tt)\leq\Psi_f(t)+O(1)$ are not automatic; they depend on whether
    the particular model can reliably perform the corresponding transformation
    from a short prompt.
\end{enumerate}

It is clear that prompt-complexity $\Psi_f$ is computable. Since $f$ is computable and always halts by Proposition
\ref{prop:f-computable}, and since $\mathcal{P}_K$ is finite, we can
enumerate all plausible prompts in nondecreasing order of length, simulate
$f_{\tau=0}$ on each prompt and return the first prompt that generates the
input text. If the finite enumeration ends without finding such a prompt, we
return $\infty$.

There are similarities between our description and time-bounded Kolmogorov
Complexity: $\Psi_f(t)$ defines a form of algorithmic text compression in which
the search space of ``programs'' is restricted to $\mathcal{P}_K$, the
interpreter is the fixed model $f$ and each evaluation is polynomially bounded
by $\mathcal{O}(K^3)$ (in the case of transformer-based models).

\begin{proposition}
    Highly prompt-compressible texts are rare.
\end{proposition}
\begin{proof}
    Fix $0<c\leq n$.
    Let $\mathcal{T}_n = \{t \in \mathcal{P}_K : |t| = n\}$ be the set of
    plausible texts of length exactly $n$, and let
    \begin{align}
        A_{n,c}=\{t\in\mathcal{T}_n:\Psi_f(t)\leq n-c\}
    \end{align}
    be the set of length-$n$ texts that save at least $c$ prompt tokens. For
    each $t\in A_{n,c}$, choose one shortest prompt $p_t\rightarrowtail t$. Since $f_{\tau=0}$ is
    deterministic, two different outputs cannot share the same chosen prompt.
    Therefore the number of such outputs is at most the number of plausible
    prompts of length at most $n-c$:
    \begin{align}
        |A_{n,c}|
        \leq
        \sum\limits_{i=0}^{n-c} |\mathcal{T}_i|.
    \end{align}
    Using the effective branching-factor approximation
    $|\mathcal{T}_i|\approx (N_{\rho,s})^i$ from the plausible-text counting
    argument, with $N_{\rho,s}>1$, this gives
    \begin{align}
        \frac{|A_{n,c}|}{|\mathcal{T}_n|}
        \lesssim
        \frac{\sum_{i=0}^{n-c}(N_{\rho,s})^i}{(N_{\rho,s})^n}
        =
        \frac{(N_{\rho,s})^{n-c+1}-1}
             {(N_{\rho,s}-1)(N_{\rho,s})^n}
        =
        O((N_{\rho,s})^{1-c}).
    \end{align}
    Hence the fraction of texts with compression gap at least $c$ decreases
    exponentially in $c$; if $c$ grows with $n$, the fraction vanishes
    exponentially in the text length.
\end{proof}

For any $x \in \mathcal{P}_K \subset \{0, 1\}^*$ with
$\Psi_f(x)<\infty$, a shortest prompt can be encoded in
$\Psi_f(x)\log_2 T+O(1)$ bits, where the constant fixes the model, tokenizer and
decoding rule. Therefore
$K(x) \leq \Psi_f(x)\log_2 T+O(1)$
\cite{li2008introduction}. Since $\mathcal{P}_K$ is finite, we also have the
uniform bound $K(x) \leq \log_2(|\mathcal{P}_K|)+O(1)$ by encoding the index of
$x$ in a fixed enumeration of plausible texts.

\begin{definition}[Conditional Prompting Complexity]
    The conditional prompting complexity $\Psi_f(u\mid v)$ for two texts $u, v$ and the LM $f$ is the length of the shortest prompt such that $f(p)$ deterministically generates $u$, while having $v$ prepended to its context.
    
    Formally, let $f:\mathcal{P}_K \to \mathcal{P}_K$ be a LM and $u, v \in \mathcal{P}_K$ two texts. Define conditional prompting complexity as:
    \begin{align}
        \Psi_f(u\mid v) = \min\{|p| : p\in \mathcal{P}_K,\ v\mathbin\Vert p\in\mathcal{P}_K,\ f_{\tau=0}(v \mathbin\Vert p) = u\}.
    \end{align}
    If the set is empty, then $\Psi_f(u\mid v)=\infty$. Here
    $v \mathbin\Vert p$ denotes concatenation, allowing for delimiter special
    tokens.
\end{definition}

The conditional prompting complexity $\Psi_f(u \mid v)$ defines the shortest prompt that transforms $v$ into $u$.

We develop a notion of prompting probability and a weak coding theorem in Appendix \ref{sec:coding-theorem}, which shows that the prompting complexity of a text is related to its probability under the model. This is a model-dependent analogue of the classical coding theorem, which links Kolmogorov complexity to algorithmic probability. The prompting probability counts the plausible prompts that generate a target text, weighted by their length, making shorter prompts contribute more mass. Unlike the universal algorithmic probability semimeasure, the prompting probability is computable in the finite LM setting, but it depends on the model and is therefore not universal.

Since prompts are themselves plausible texts, a shortest prompt
$\phi_f(t)\rightarrowtail t$ may also be studied as an output target. The quantity
$\Psi_f(\phi_f(t))$ measures how hard it is for the same model to generate a
prompt that generates $t$. There is no general monotonic relation between
$\Psi_f(t)$ and $\Psi_f(\phi_f(t))$: the shortest prompt for $t$ may be
unreachable, or may itself have a shorter, equal-length, or longer shortest
prompt.

The same issue appears even more directly when comparing different models.
Prompting complexity counts only the visible prompt, not the information stored
in the model. Thus a long text may be cheap for one model because it is
memorized, retrievable or strongly implied by a short instruction, while the
same text may be expensive for another model that lacks that shortcut. The
simplest reason for model dependence is therefore that the model acts as the
decompressor, and its parameters are supplied for free.

\subsection{Model Dependence of Prompting Complexity}
\begin{proposition}[Failure of invariance]
Let \(f : \mathcal{P}^f_K \to \mathcal{P}^f_K\) and \(g : \mathcal{P}^g_K \to \mathcal{P}^g_K\) be two LMs.
There is no analogue of the invariance theorem for prompting complexity. More
precisely, for every constant \(C > 0\), for sufficiently large context size \(K\),
there exist two LMs \(f,g\) and a plausible text
\(t \in \mathcal{P}^f_K \cap \mathcal{P}^g_K\) such that $\left|\Psi_f(t) - \Psi_g(t)\right|> C$.

Therefore prompting complexity is not universal: the complexity assigned to a text
depends essentially on the underlying LM.
\end{proposition}

\begin{proof}
Fix a LM \(g\). For each \(n \leq K\), let $\mathcal{P}^g_n := \{t \in \mathcal{P}^g_K : |t| = n\}$ be the set of plausible texts of length exactly \(n\) under \(g\). By Proposition 2, the number of plausible texts grows exponentially with \(n\). Hence, for some effective branching factor \(N_g > 1\), $|\mathcal{P}^g_n| \asymp N_g^n$.

Now consider the set of length-\(n\) texts that are \(c\)-prompt-compressible under \(g\):
\begin{align}
    A^g_{n,c} := \left\{t \in \mathcal{P}^g_n : \Psi_g(t) < n-c\right\}
\end{align}
Since \(g\) is deterministic at temperature \(\tau = 0\), each prompt produces at
most one output. Therefore the number of texts in $A^g_{n,c}$ is bounded above
by the number of plausible prompts of length strictly less than \(n-c\): $|A^g_{n,c}|\leq \sum_{i=0}^{n-c-1} |\mathcal{P}^g_i|$.

Using the exponential growth assumption again,
\begin{align}
|A^g_{n,c}|
\leq
C_1 \sum_{i=0}^{n-c-1} N_g^i
=
C_1 \frac{N_g^{n-c}-1}{N_g-1}
=
O(N_g^{n-c})
\end{align}
for some constant \(C_1>0\). On the other hand, $|\mathcal{P}^g_n| \asymp N_g^n$. Thus, for sufficiently large fixed \(c\), $|A^g_{n,c}| < |\mathcal{P}^g_n|$. Consequently, there exists at least one plausible text \(t_n \in \mathcal{P}^g_n\) such that $t_n \notin A^g_{n,c}$.

By definition of \(A^g_{n,c}\), this means that \(t_n\) is not generated by any prompt of length less than \(n-c\) under \(g\). Therefore $\Psi_g(t_n) \geq n-c$.

Now construct another LM \(f\) which has \(t_n\) encoded in its parameters, for example because \(t_n\) occurred in its pretraining data or was memorized during fine-tuning, and suppose that \(t_n\) can be elicited by a fixed constant-length prompt \(q_n\), such as an identifier or retrieval instruction: $f_{\tau=0}(q_n) = t_n, |q_n| = O(1)$. Then $\Psi_f(t_n) \leq |q_n| = O(1)$. Combining the two inequalities gives:
\begin{align}
\Psi_g(t_n) - \Psi_f(t_n) \geq n-c-O(1)
\end{align}
Since \(n\) can be chosen arbitrarily large up to the context size \(K\), for every
constant \(C\) we may choose \(n\) large enough so that $n-c-O(1) > C$.

Hence,
\begin{align}
    \left|\Psi_f(t_n)-\Psi_g(t_n)\right| > C.
\end{align}
Therefore no model-independent additive constant can relate \(\Psi_f\) and
\(\Psi_g\) in the way required by the classical invariance theorem for Kolmogorov
complexity. Prompting complexity is intrinsically model-dependent.

\end{proof}

Classical Kolmogorov complexity obtains invariance by allowing one universal machine to simulate another with a
fixed compiler overhead. A prompt does not, in general, contain such a compiler
from one LM to another, nor can it transfer the source model's memorized texts,
learned associations or decoding behavior. Changing \(f\) changes which
information is available at zero prompt cost, so no model-independent additive
constant can control \(\Psi_f(t)\) across all LMs.

\section{Soft Prompting Complexity}
\label{sec:soft-prompting-complexity}
Usually, prompt engineering requires finding a prompt that leads to the generation of a desirable text, by some definition of "desirable". From a practical standpoint, we could argue that the best such prompt is the shortest one, as short prompts are more efficient to compute (in both memory and time), and, for example, reduce inference costs if using a commercial model. The quantity $\Psi_f$ is too restrictive in this sense, since exact matching is infeasible in practice. Naturally, we can relax the definition of $\Psi_f$ by taking into account a distance function between texts and a threshold for how close the generated text is to a desired outcome. The distance measure between texts could be chosen as an edit distance \cite{levenshtein1966binary} or semantic similarity \cite{chandrasekaran2021evolution} using another pretrained model. 

\begin{definition}[Soft Prompting Complexity]
    \label{def:epsilon-prompt-complexity}
    Let $f: \mathcal{P}_K \rightarrow \mathcal{P}_K$ be a LM and $t \in \mathcal{P}_K$. Let $d: \mathcal{P}_K \times \mathcal{P}_K \rightarrow \mathbb{R}_{\geq 0}$ be a distance function and $\varepsilon > 0$. Define $\Psi^{\varepsilon,d}_{f}: \mathcal{P}_K \rightarrow \mathbb{N}\cup\{\infty\}$ as: 
    \begin{equation}
        \Psi^{\varepsilon,d}_f(t) := \min_{p\in \mathcal{P}_K}\{|p|~: d(f_{\tau = 0}(p),  t) < \varepsilon\}
    \end{equation}
    If there is no such $p \in \mathcal{P}_K$, then $\Psi^{\varepsilon,d}_f(t)=\infty$. Mirroring the notation in Definition~\ref{def:prompt-complexity}, we write \((p^* \overset{\varepsilon,d}{\rightarrowtail}_f t)\), or simply \((p^* \overset{\varepsilon,d}{\rightarrowtail} t)\) when \(f\) is fixed, to mean that $p^*$ is a shortest prompt such that $d(f_{\tau = 0}(p^*),  t) < \varepsilon$.
\end{definition}

\begin{proposition}
    The Soft Prompting Complexity is a relaxed form of Prompting Complexity. In particular, for any distance function $d$, $t\in\mathcal{P}_K$, and $\varepsilon>0$:
    \begin{align}
        \Psi_f^{\varepsilon,d}(t) \leq \Psi_f(t)
    \end{align}
\end{proposition}
\begin{proof}
    If there is no prompt $p$ such that $f_{\tau=0}(p) = t$, by Definition \ref{def:prompt-complexity} we have that $\Psi_f(t) = \infty$, so the inequality is immediate. Assume instead that $p^*\rightarrowtail t$. Since $d$ is a distance, $d(f_{\tau=0}(p^*), t)=0<\varepsilon$, so $p^*$ is feasible for $\Psi_f^{\varepsilon,d}(t)$, but not necessarily of minimum length among all prompts satisfying the relaxed condition. Thus $\Psi_f^{\varepsilon,d}(t) \leq |p^*| = \Psi_f(t)$ for all $\varepsilon > 0$.
\end{proof}

The function $\Psi^{\varepsilon,d}_f(t)$ is defined with respect to a distance measure and threshold $\varepsilon$ and generalizes prompting complexity to a lossy setting, analogously to rate–distortion in information theory \cite{shannon1948info}. This naturally gives rise to a (gradient-free) optimization algorithm which can be used to approximate a relaxed shortest prompt $p^*\overset{\varepsilon,d}{\rightarrowtail} t$ \cite{chouayfati2025gendln,khattab2024dspy,yuksekgonul2025optimizing}. This relaxed notion of prompting complexity captures the trade‑off between prompt brevity and output fidelity, offering a principled way to design "good enough" prompts that could be shorter than those needed for perfect reproduction.

Since $\mathcal{P}_K$ is finite, this relaxation also recovers ordinary prompting complexity in the zero-distortion limit. If $d$ is a metric, then for sufficiently small $\varepsilon$ the only feasible outputs are exact copies of $t$; hence $\lim_{\varepsilon\to0^+}\Psi_f^{\varepsilon,d}(t)=\Psi_f(t)$, with both sides interpreted in $\mathbb{N}\cup\{\infty\}$.

Having defined the more tractable quantity of Soft Prompting Complexity, we can then define a related way to compare texts. A long-standing difficulty in text similarity is computing a semantic distance between long texts, where practical systems often resort to comparing chunks, summaries, or external embeddings. Here, we define the distance between two texts as the distance between their \textit{causes}: the shortest prompts that approximately generate them under a fixed LM.

\subsection{Prompting Distance}
\label{sec:prompting-distance}
\begin{definition}[Prompting Distance]
    \label{def:prompting-distance}
    The prompting distance of two texts is the distance between the shortest prompts that approximately generated them.

    Formally, fix $\varepsilon>0$ and two distances between texts $d, d': \mathcal{P}_K \times \mathcal{P}_K \rightarrow \mathbb{R}_{\geq 0}$. 
    For any $t_1,t_2\in\mathcal{P}_K$ with finite $\Psi_f^{\varepsilon,d'}(t_1)$ and $\Psi_f^{\varepsilon,d'}(t_2)$, having $(p_1^* \overset{\varepsilon,d'}{\rightarrowtail} t_1)$ and $(p_2^* \overset{\varepsilon,d'}{\rightarrowtail} t_2)$, we define the Prompting Distance $d_{\Psi}$ between $t_1$ and $t_2$ as:
    \begin{align}
         d_{\Psi}(t_1,t_2) = d(p^*_1,p^*_2).
    \end{align}
\end{definition}

This construction separates similarity of effects from similarity of causes. Two texts can be close in ordinary text space while having very different shortest prompts. For example, if LM generations exhibit ergodic behavior on long sequences, they may converge toward similar outputs despite very different initial prompts \cite{zekri2024large}. Conversely, two outputs can be far apart in surface or semantic space while their shortest prompts differ only by a small edit, such as changing the requested topic, stance, or target object. Thus $d_{\Psi}$ captures similarity of the generating conditions, not only similarity of the generated artifacts.

This is useful for prompt optimization because the search is performed in prompt space while the objective is measured in output space. Prompt optimizers such as those cited above implicitly try to infer which prompt edits caused useful output changes. The prompting distance makes this dependence explicit: prompts that are near under $d$ but lead to distant outputs indicate sensitive directions in prompt space, while distant prompts that lead to similar outputs indicate redundant or convergent prompting strategies.

The prompting distance is similar to "description similarity" \cite{ravfogel2023description}, but differs as it considers the generating prompts for texts and not descriptions of texts, which can be regarded as summaries. It is also related to the trajectory-based view of meaning representations in autoregressive models \cite{liu2023meaning}. Their perspective uses model trajectories to study representations associated with meaning; here, by contrast, we compare texts by first inverting them to short generating prompts and then comparing those prompts.

\section{The Prompt Engineering Problem: Behavioral Prompting Complexity}
\label{sec:behavioral-complexity}
Exact prompting complexity starts from a single target text: how short can a
prompt \(p\rightarrowtail t\) be if it has to make the model output exactly
\(t\)? The \((\varepsilon,d)\)-version replaces this with
\(p\overset{\varepsilon,d}{\rightarrowtail}t\), allowing outputs that are close
to \(t\). Many practical questions are still broader than this. We often do not
care which exact string the model produces, as long as the output has some
property. 

The useful object is therefore not a single text, but a set of acceptable
texts. Here, we consider the use of an LM-as-a-Judge \cite{gu2026survey} to determine whether an output satisfies a given behavioral specification.

\begin{definition}[Behavioral Specification]
    \label{def:behavioral-specification}
    A behavioral specification induced by a judge prompt \(j\) is a subset
    \(B_j\subseteq\mathcal{P}_K\) of plausible outputs accepted by an LM
    judge with prompt \(j\). We write \(u\models B_j\) when \(u\in B_j\).
    
    A judge is a function
    \(J_j:\mathcal{P}_K\rightarrow\{0,1\}\) that
    returns \(1\) when a candidate output satisfies the desired behavior
    specified in the prompt \(j\).
    \begin{align}
        B_j
        :=
        \{u\in\mathcal{P}_K: j\mathbin\Vert u\in\mathcal{P}_K,\ J_j(j\mathbin\Vert u)=1\}.
    \end{align}
Thus \(B_j\) is the set of outputs that pass the judge's evaluation. The
judge prompt \(j\) can be thought of as a description of the desired behavior,
and the judge function \(J_j\) evaluates whether a given output meets that
behavior.
\end{definition}

\begin{definition}[Behavioral Prompting Complexity]
    \label{def:behavioral-prompting-complexity}
    Let \(f:\mathcal{P}_K\rightarrow\mathcal{P}_K\) be an LM and let
    \(B\subseteq\mathcal{P}_K\) be a behavioral specification. The behavioral
    prompting complexity of \(B\) is the length of the shortest plausible prompt
    whose deterministic output satisfies \(B\):
    \begin{align}
        \Psi_f(B)
        :=
        \min_{p\in\mathcal{P}_K}
        \{|p|: f_{\tau=0}(p)\models B\}.
    \end{align}
    If no plausible prompt reaches \(B\), then \(\Psi_f(B)=\infty\).
\end{definition}

Intuitively, \(\Psi_f(B)\) asks how hard it is to land anywhere inside the
accepted region \(B\). If
\[
    \operatorname{Pre}_f(B)
    =
    \{p\in\mathcal{P}_K:f_{\tau=0}(p)\in B\}
\]
is the set of plausible prompts that reach \(B\), then \(\Psi_f(B)\) is just the
length of the shortest prompt in this preimage. This is a controllability
measure for a behavior class, rather than a compressibility measure for one
specified output.

\begin{proposition}
    Behavioral prompting complexity generalizes both prompting complexity and
    \((\varepsilon,d)\)-prompting complexity. In particular:
    \begin{align}
        \Psi_f(\{t\}) &= \Psi_f(t),\\
        \Psi_f(A_{\varepsilon,d}(t)) &= \Psi_f^{\varepsilon,d}(t),
    \end{align}
    where
    \(A_{\varepsilon,d}(t)=\{u\in\mathcal{P}_K:d(u,t)<\varepsilon\}\).
    Moreover,
    \begin{align}
        \Psi_f(B)=\min_{u\in B}\Psi_f(u),
    \end{align}
    with the convention that the minimum over the empty set is \(\infty\).
\end{proposition}
\begin{proof}
    A singleton behavior \(B=\{t\}\) asks for exactly the prompts that produce
    \(t\), whose minimizers satisfy \(p\rightarrowtail t\), so it recovers \(\Psi_f(t)\). If \(B\) is the ball
    \(A_{\varepsilon,d}(t)\), then reaching \(B\) means producing an output
    within distance \(\varepsilon\) of \(t\), which is exactly Definition
    \ref{def:epsilon-prompt-complexity}.

    For a general \(B\), every successful prompt produces some accepted output
    \(u\in B\). The shortest way to reach \(B\) is therefore the shortest among
    all the shortest ways to produce one of its members:
    \[
        \Psi_f(B)=\min_{u\in B}\Psi_f(u).
    \]
\end{proof}

The next property says that making the accepted set larger cannot make
the task harder. If more outputs count as success, the shortest successful
prompt can only get shorter, or stay the same.

\begin{proposition}[Monotonicity]
    Having a larger behavioral specification cannot increase the behavioral prompting complexity. In particular, if \(B_1\subseteq B_2\subseteq\mathcal{P}_K\), then $\Psi_f(B_2)\leq \Psi_f(B_1)$.
\end{proposition}
\begin{proof}
    Every prompt that reaches \(B_1\) also reaches \(B_2\). Thus the set of
    feasible prompts for \(B_1\) is contained in the feasible set for \(B_2\),
    and minimizing prompt length over the larger set cannot give a larger
    value. 
\end{proof}

The judge prompt \(j\) affects behavioral prompting complexity only through
the set \(B_j\) that it defines. Thus \(j\) is a description of an accepted
region in output space, while \(\Psi_f(B_j)\) measures how hard it is for the
generator \(f\) to reach that region. In general, there is no monotone
relationship between \(|j|\) and \(\Psi_f(B_j)\): a short judge prompt can
define an unreachable behavior, while a long rubric can define an easy one.
For example, a short specification such as ``solve this open problem'' may
describe a very small or unreachable set of acceptable answers. Conversely, a
long grading rubric may accept many ordinary outputs, making the corresponding
behavior easy to reach.

\begin{definition}[Conditional Behavioral Prompting Complexity]
    Let \(B_j\subseteq\mathcal{P}_K\) be the behavioral specification induced
    by the judge prompt \(j\). The conditional behavioral prompting complexity
    of \(B_j\) given \(j\) is
    \begin{align}
        \Psi_f(B_j\mid j)
        :=
        \min_{p\in\mathcal{P}_K}
        \{|p|:j\mathbin\Vert p\in\mathcal{P}_K,\ f_{\tau=0}(j\mathbin\Vert p)\models B_j\}.
    \end{align}
    If no such prompt exists, then \(\Psi_f(B_j\mid j)=\infty\).
\end{definition}

This conditional quantity separates the cost of stating the behavioral
specification from the cost of satisfying it. The judge prompt \(j\) is treated
as already present in the context, and \(\Psi_f(B_j\mid j)\) measures the
additional prompt length needed to make the generator produce an output that
would pass the judge. In prompt-engineering terms, \(j\) is the task or rubric,
while \(p\) is the remaining instruction that tries to make the model comply
with that rubric.

\begin{proposition}[Specification-conditioning bound]
    For any behavioral specification \(B_j\) induced by a judge prompt \(j\),
    \begin{align}
        \Psi_f(B_j)
        \leq
        |j|+\Psi_f(B_j\mid j)+O(1),
    \end{align}
    whenever the right-hand side is finite.
\end{proposition}
\begin{proof}
    Let \(p^*\) be a shortest prompt for \(\Psi_f(B_j\mid j)\). By
    definition, \(j\mathbin\Vert p^*\in\mathcal{P}_K\) and
    \(f_{\tau=0}(j\mathbin\Vert p^*)\models B_j\). Hence
    \(j\mathbin\Vert p^*\) is a feasible prompt for \(\Psi_f(B_j)\). Its
    length is \(|j|+|p^*|+O(1)\), where the constant accounts for fixed
    delimiter or template tokens. Therefore
    \(\Psi_f(B_j)\leq |j|+\Psi_f(B_j\mid j)+O(1)\).
\end{proof}

The bound says that an unconditional prompt can always spend tokens to include
the specification \(j\), and then append whatever conditional prompt would work
once \(j\) is in context. It is only an upper bound: the shortest prompt for
reaching \(B_j\) need not literally contain the judge prompt. The model may
know a much shorter way to reach the same accepted region.

\begin{proposition}[Direct specification bound]
    Suppose there is a fixed instruction \(a\in\mathcal{P}_K\), independent of
    \(j\), such that \(j\mathbin\Vert a\in\mathcal{P}_K\) and
    \(f_{\tau=0}(j\mathbin\Vert a)\models B_j\) for every judge prompt \(j\)
    in a class of interest. Then, for every such \(j\),
    \begin{align}
        \Psi_f(B_j)\leq |j|+O(1).
    \end{align}
\end{proposition}
\begin{proof}
    The concatenated prompt \(j\mathbin\Vert a\) is feasible for
    \(\Psi_f(B_j)\) by assumption. Since \(a\) is fixed, its length and the
    delimiter overhead are absorbed into the \(O(1)\) term.
\end{proof}

This special case captures the idealized situation in which the model can
reliably turn a behavioral specification into a satisfying output using a fixed
generic instruction such as ``produce an answer satisfying the above
criterion.'' Then the cost of controlling the behavior is at most the cost of
writing down the specification, up to fixed template overhead. In practice this
assumption is strong: failures of instruction following, ambiguity in \(j\), or
an empty accepted set \(B_j\) can all break the bound's premise.

\section{Discussion and Related Work}
\label{sec:related}
\subsection{Prompt Search and Optimization}

Prompt engineering is usually presented as the problem of finding an input that
causes a model to behave well on a task. Early prompting work showed that large
language models can be controlled by natural-language instructions, few-shot
examples, and task framing rather than by weight updates
\cite{brown2020language,reynolds2021prompt,schulhoff2024prompt}. Chain-of-thought
prompting further demonstrated that the prompt may specify not only the desired
answer but also the form of intermediate computation to be performed
\cite{wei2022chain}. In this literature, the prompt is an interface for
controlling a model; in our formalism, it is also a candidate compressed
description of the output.

Automatic prompt optimization makes this search explicit. AutoPrompt searches
for trigger tokens that elicit knowledge from masked language models
\cite{shin2020autoprompt}; APO uses natural-language critiques as a surrogate
for gradients and combines prompt edits with beam search
\cite{pryzant2023automatic}; DSPy treats language-model calls as optimizable
modules in a larger program \cite{khattab2024dspy}; and recent evolutionary or
feedback-based approaches optimize prompt-like text through black-box
evaluation or language-model feedback
\cite{chouayfati2025gendln,yuksekgonul2025optimizing}. These methods estimate
useful prompts under a task metric. The quantity
$\Psi_f^{\varepsilon,d}(t)$ gives a complementary idealization: among all
plausible prompts that produce an acceptable output, it asks for the shortest
one. Practical prompt optimizers can therefore be viewed as algorithms for
finding upper bounds on relaxed prompting complexity.

This also distinguishes our setting from analyses of prompt-space complexity
\cite{zhang2025prompt}. Work in that direction studies how the size and
structure of the prompt space affects reasoning performance, especially for
chain-of-thought prompting. We instead assign a complexity to a target text
relative to a fixed model and decoding rule. The two views meet in the search
problem: a difficult prompt space makes it hard to find a relaxed shortest
prompt $p^*\overset{\varepsilon,d}{\rightarrowtail}t$, even when such a prompt
exists.

\subsection{Prompt Inversion}

Prompting complexity turns prompt engineering around. The forward problem is:
given a prompt, evaluate the output. The inverse problem is: given an output or
behavior, find a short plausible prompt that caused it. Hu et al.
\cite{hu2023amortizing} study intractable posterior inference in language
models, including latent-variable views of chain-of-thought reasoning. In our
case the latent variable is the prompt itself: for a target text $t$, the
preimage set $\{p\in\mathcal{P}_K:f_{\tau=0}(p)=t\}$ is the object being
searched, and its shortest elements satisfy $p\rightarrowtail t$.

This should not be confused with model stealing or with inversion of hidden
states. Recent results on injectivity argue that transformer language models
can preserve enough information in continuous hidden representations to recover
discrete inputs \cite{nikolaou2025language}. Surjectivity results ask whether
neural networks can in principle realize arbitrary outputs from some input
\cite{jiang2025surjectivity}. Prompting complexity is defined at the visible
text interface after templating, decoding, and detokenization. At that
interface many prompts may lead to the same text, and many texts may be
unreachable by any plausible prompt. The relevant safety and usability question
is therefore not only whether an inverse exists somewhere in token space, but
whether a short human-plausible inverse exists.

Although $\Psi_f$ is computable by exhaustive enumeration in the finite setting,
the search is intractable in practice. This gives the quantity a flavor closer
to resource-bounded and instance-level complexity than to a usable compression
algorithm \cite{orponen1994instance}. A target text may be easy to verify once a
candidate prompt is found, yet hard to invert by search. In this weak sense,
some outputs behave like model-relative hashes of their prompts: the forward
map is cheap to evaluate, while recovering a short cause may require searching
an exponentially large plausible prompt space. This analogy is only heuristic,
since we make no cryptographic assumption about $f$.

\subsection{Compression, Context, and Tokenization}

Our notion of compression differs from ordinary language-model compression.
Predictive language models can be turned into strong lossless compressors
\cite{llm2024compression}, and classical compression studies the number of bits
needed to encode a string given an encoder and decoder. Prompting complexity
instead asks how many prompt tokens are needed to make a fixed model regenerate
or approximate a text. The decoder is not a general-purpose decompressor shared
by all users; it is the particular LM, with its training distribution,
tokenizer, instruction tuning, system prompt, and decoding rule already baked
in. This is why the failure of an invariance theorem is not a defect but a
feature of the definition.

The distinction also separates our framework from context compression. Context
pruning and related methods remove uninformative tokens from an existing
context to reduce inference cost \cite{anagnostidis2023dynamic}. Prompting
complexity asks for a new cause of the output, not a shorter subset of the
original context. Similarly, model representations and routing work can build
compact descriptors of language models themselves \cite{zhuang2024embedllm,
jitkrittum2025universal}, while tokenizer research studies how raw strings are
mapped into token sequences and how those choices affect efficiency and
statistical consistency \cite{gastaldi2024foundations,minixhofer2024zero}.
Those factors matter because $\Psi_f$ is measured in tokens and is therefore
tokenizer-dependent, but the object being minimized is still a plausible prompt.

There is a useful analogy to grammar-based compression. The smallest grammar
problem asks for the smallest grammar that generates a given string
\cite{charikar2005smallest}. A shortest prompt is likewise a compact
generative description, but with two important differences: the "grammar" is
the learned model plus decoding rule, and the description must be a plausible
text rather than an arbitrary formal object. This plausibility constraint is
what makes the quantity relevant to prompt engineering rather than merely to
coding theory.

\subsection{Synthetic Data and Reverse Instructions}

The framework clarifies a common ambiguity around synthetic data. If a dataset
$D$ is generated by a fixed model from prompts $p_1,\ldots,p_n$, then the
synthetic corpus does not contain arbitrary new information relative to that
model; much of it is already compressed by the prompts, the sampling rule, and
the model weights. In this sense, a large synthetic dataset can have a short
model-relative description even when its surface form is long. This does not
make synthetic data useless. It can reweight modes of the model distribution,
make implicit capabilities explicit, expose rare combinations of skills, or
provide training examples in a format that another model can consume. But its
information content should be discussed relative to the generator and the
selection process, not just by counting output tokens.

Reverse-instruction methods are a concrete example. LongForm constructs
instruction-tuning data by taking human-written documents and generating
instructions that could have produced them \cite{koksal2023longform}. This is
an approximate inversion step: the document is treated as the target output,
and a model proposes a shorter natural-language cause. The resulting
instruction-output pair can then train another model to make long outputs more
accessible from short prompts. From the perspective of prompting complexity,
such data may reduce $\Psi_g^{\varepsilon,d}(t)$ for a downstream model $g$,
even if it does not create new information ex nihilo.

In-context learning provides another route by which prompts change the effective
behavior of a model without changing its stored weights. Few-shot prompting
\cite{brown2020language} and recent analyses of implicit in-context dynamics
\cite{dherin2025learningtrainingimplicitdynamics} suggest that examples in the
prompt can act like temporary task-specific updates. Conditional prompting
complexity $\Psi_f(u\mid v)$ captures this effect at the text interface: the
same target $u$ may require a long prompt in isolation but a much shorter prompt
when examples, demonstrations, or a domain context $v$ are already present.

\subsection{Reasoning Traces and Output Length}

Reasoning models complicate the accounting because the prompt may elicit a long
internal or visible trace before the final answer. Chain-of-thought prompting
\cite{wei2022chain}, DeepSeek-R1 \cite{guo2025deepseek}, and the OpenAI o1
system card \cite{jaech2024openai} all point to a regime where useful behavior
depends on extra inference-time computation. In this paper, the primary object
whose complexity is measured is the final output text, while any reasoning
trace is treated as part of the model's bounded computation unless it is made
part of the target. A natural extension would add a cost term for reasoning
tokens or wall-clock inference, giving a speed-prior analogue \cite{schmidhuber2000algorithmic} 
of prompting complexity.

Even if two prompts have the same length and produce
equally good final answers, one may require a much longer generation. Work on
forecasting output length \cite{piotrowski2025will} is therefore relevant to
practical prompt search: estimating whether a candidate prompt is cheap enough
to evaluate can matter as much as estimating whether it will succeed.

\subsection{Safety and Jailbreaks}

The same language gives a compact formulation of one part of AI safety. Let
$H\subseteq\mathcal{P}_K$ be a set of harmful or otherwise undesirable texts.
A model is safer, with respect to this interface, when every $t\in H$ has large
or infinite relaxed prompting complexity: no short plausible prompt should
produce an unacceptable output. This differs from making harmful continuations
low probability under ordinary sampling. The relevant adversary is not sampling
randomly from the model, but searching for a prompt in the preimage of $H$.

The plausibility restriction is important in this context. Some jailbreaks and glitch prompts
are visibly unnatural token artifacts \cite{rumbelow2023solidgoldmagikarp,
plinius}; these may be excluded by a sufficiently strict plausibility filter.
Other attacks are ordinary-looking instructions, role plays, or contextual
setups. Such prompts remain inside $\mathcal{P}_K$, so perplexity or
token-level plausibility filters are insufficient. Alignment and instruction
following \cite{ouyang2022traininglanguagemodelsfollow} can be understood as
reshaping the map $f$ so that harmful texts have smaller preimage mass
$m_f(t)$, longer shortest prompts, or no plausible prompts at all. The
surjectivity perspective \cite{jiang2025surjectivity} suggests why this is
difficult: the goal is not merely to remove all mathematical preimages, but to
make dangerous preimages inaccessible under realistic prompt constraints.

\subsection{Similarity, Meaning, and Causes}

Prompting distance compares texts by comparing short causes rather than only
surface forms or embeddings. This is related to description-based text
similarity \cite{ravfogel2023description}, conceptual similarity through
complexity-constrained descriptions \cite{achille2024interpretable}, and
trajectory-based representations of meaning in autoregressive models
\cite{liu2023meaning}. The difference is that a prompting-distance prompt
$p\overset{\varepsilon,d}{\rightarrowtail}t$ is not merely a description or
summary of the text; it is an input that actually causes the fixed model to
generate an acceptable output.

This cause-based view explains why semantic similarity of outputs need not
imply similarity of prompts. Two paraphrases may be generated from entirely
different instructions, while two very different outputs may differ only by a
small edit in the prompt, such as changing a topic, style, or target entity.
Thus $d_\Psi$ is best viewed as a model-relative geometry of controllability:
it describes which changes in the prompt space are small causes of large or
small changes in output space.

\section{Conclusions}
\label{sec:conclusions}
We introduced prompting complexity as a model-dependent analogue of
algorithmic text complexity for pretrained instruction-tuned language models.
The central question is simple: for a fixed LM $f$, what is the shortest
plausible prompt that causes deterministic decoding to produce a target text
$t$, i.e. a prompt $p\rightarrowtail t$? This changes the usual view of prompting from a craft of phrasing
instructions into a formal search problem over human-interpretable causes.
When a short prompt elicits a long text, the missing information is supplied by
the model's weights, training distribution, tokenizer, prompt template, and
decoding rule.

The resulting quantity $\Psi_f(t)$ inherits part of the intuition of
Kolmogorov complexity while departing from it in the ways that matter for
language models. Since a LM has finite context and always halts,
prompting complexity is computable in principle. Since prompts are restricted
to plausible texts and the model is fixed, it is not universal: there is no
model-independent invariance theorem. A text can be highly compressible for one
model because it is memorized, represented, or easy to elicit, while being
nearly inaccessible to another. Prompting complexity is therefore not an
absolute property of a text, but a property of a text relative to a particular
language-model interface.

We also defined a prompting probability and proved a weak coding theorem that
relates prompt length, prompt preimage mass, and ordinary model probability.
This connects the search for short prompts to the probability mass assigned by
the model to texts and to the multiplicity of prompts that produce the same
output. In the relaxed setting, $(\varepsilon,d)$-prompting complexity captures
the practical prompt-engineering objective: finding a short prompt whose output
is close enough to a desired target under a chosen distance. The idealized
optimum is a relaxed shortest prompt
$p\overset{\varepsilon,d}{\rightarrowtail}t$. This turns prompt
optimization into lossy, model-relative generative compression.

The broader value of the framework is that it gives a common language for
phenomena that are usually discussed separately. Prompt optimization becomes
approximate search for short prompts. Prompt inversion asks for plausible
preimages of observed completions. Synthetic data can be analyzed by the
length of the prompts and generation process that describe it relative to the
generator. Safety can be phrased as making harmful texts prompt-incompressible
or at least hard to reach from short plausible prompts. Text similarity can be
studied through the similarity of causes rather than only the similarity of
surface outputs.

This view suggests that prompts
are model-relative descriptions of possible texts and behaviors. As language
models become more capable, longer-context, and more explicitly equipped with
reasoning traces, understanding which outputs are accessible from which short
plausible prompts becomes increasingly important. Prompting complexity offers a formal vocabulary for that question.

\section*{Limitations}
\label{sec:limitations}
The main limitation is computational. Exact prompting complexity is finite and
computable only because the model, context window, tokenizer, and plausible
text set are finite; exhaustive search is still infeasible. Any empirical
method can at best provide upper bounds by finding a short prompt, lower bounds
by ruling out restricted search regions, or probabilistic estimates through
sampling. 

Finally, low prompting complexity is not the same as low problem complexity. A
math solution, legal argument, or factual report can have a short prompt if the
model already contains the relevant structure, while a trivial-looking output
can be hard to elicit exactly because of decoding quirks or instruction-tuning
priors. Prompting complexity measures accessibility of a text through a
particular model interface. That makes it useful for analyzing prompt
optimization, inversion, synthetic data, safety, and similarity, but it also
marks the boundary of the concept.

\bibliographystyle{plainnat}
\bibliography{refs}

\section*{Appendix}
\label{sec:appendix}
\subsection{LMs are computable and polynomially bounded}

\begin{proposition}
\label{prop:f-computable}
A LM $f$ is computable, and it is polynomially space- and time-bounded by $\mathcal{O}(K^3)$, with $K$ the maximal sequence size for transformer inference. 
\end{proposition}
\begin{proof}
    By definition, the LM $f$ is constructed through a composition of computable functions, has fixed-precision parameters and has a fixed context size $K$. Therefore $f$ is computable\footnote{Or simply: since we can run an LM on a computer, it is computable.} and always halts. Since the computational complexity of a single forward pass (i.e., for generating a single token) of a transformer decoder is $\mathcal{O}(n^2)$ \cite{keles2023computational}, due to quadratic complexity of self-attention, generating a text of size $|t|$ from a prompt of size $|p|$ is then:
    \begin{align}
        \sum\limits_{i = 0}^{|t|} \mathcal{O}((|p| + i)^2)
        &= \mathcal{O}\left(|t||p|^2 + |p|(|t|^2 + |t|) + \frac{|t|(|t| + 1)(2|t| + 1)}{6}\right) \\
        &= \mathcal{O}((|p| + |t|)^3).
    \end{align}
    Since the active sequence length is at most $K$, the runtime is bounded by $\mathcal{O}(K^3)$. The memory used by the finite context and attention computation is also polynomial in $K$, and in particular upper-bounded by $\mathcal{O}(K^3)$ under this loose accounting.
\end{proof}

\subsection{Additional propositions on plausible texts}

\begin{proposition}[Number of Plausible Texts]
For fixed $f,\tau,\rho$ and context $c$, let
$\mathcal{P}^{\tau,\rho}_{n}(c)$ be the set of plausible continuations of
length exactly $n\leq K-|c|$. Then
\begin{align}
    |\mathcal{P}^{\tau,\rho}_{n}(c)|\leq \ceil{\rho T}^{\,n}.
\end{align}
Under the effective-branching approximation
$|S_\rho(c\mathbin\Vert t_{<i})|\approx N_{\rho,s}>1$ along plausible prefixes,
the number of plausible texts grows exponentially:
$|\mathcal{P}^{\tau,\rho}_{n}(c)|\approx (N_{\rho,s})^n$.
\end{proposition}

\begin{proof}
    For any probability distribution over $T$ tokens sorted in nonincreasing
    order, the mass of the top $k$ tokens is at least $k/T$. Thus the top
    $\ceil{\rho T}$ tokens have total mass at least $\rho$, so
    $|S_\rho(u)|=N_\rho(u)\leq \ceil{\rho T}$ for every prefix $u$.
    The plausible length-$n$ continuations form a rooted tree whose branching
    factor at every depth is at most $\ceil{\rho T}$. Hence the number of
    leaves is at most $\ceil{\rho T}^{\,n}$.

    To obtain a typical exponential growth rate, assume that along plausible
    prefixes the next-token probabilities approximately follow a power law
    \footnote{This is a simplification. In practice the exponent and even the
    admissible vocabulary change with the context, the tokenizer and syntactic
    constraints such as JSON formatting.}
    \begin{align}
        P(v_{(i)}\mid u)\approx \frac{i^{-s}}{\zeta(s)},
        \qquad s>1,
    \end{align}
    with large vocabulary size $T$\footnote{This is usually the case. Models
    such as Gemma3 \cite{team2025gemma} have vocabulary sizes of 256,000
    tokens.}. The nucleus size is approximately the smallest $k$ satisfying
    $\sum_{i=1}^k i^{-s}\geq \rho\zeta(s)$. Using the tail approximation
    $\sum_{i=1}^k i^{-s}\approx \zeta(s)-k^{1-s}/(s-1)$ gives
    \begin{align}
        \zeta(s)-\frac{k^{1-s}}{s-1}\geq \rho\zeta(s)
        \quad\Longleftrightarrow\quad
        k\gtrsim ((s-1)(1-\rho)\zeta(s))^{-\frac{1}{s-1}}.
    \end{align}
    We write
    \begin{align}
        N'_{\rho,s}
        :=
        ((s-1)(1-\rho)\zeta(s))^{-\frac{1}{s-1}},
        \qquad
        N_{\rho,s}
        :=
        \min\{\ceil{\rho T},\max\{1,N'_{\rho,s}\}\}.
    \end{align}
    If this effective nucleus size is roughly stable across prefixes, then the
    branching process has about $N_{\rho,s}$ choices per token and
    $|\mathcal{P}^{\tau,\rho}_{n}(c)|\approx (N_{\rho,s})^n$.
\end{proof}

\begin{proposition}
Greedy outputs are plausible for every nucleus threshold. More precisely, fix
$\rho\in(0,1)$ and let $t=f_{\tau=0}(c)$ be the deterministic greedy completion
from context $c$, truncated to fit the context window. Then
$t\in\mathcal{P}^{\tau,\rho}_K(c)$ for every $\tau>0$.
\end{proposition}
\begin{proof}
At any generation step, greedy decoding chooses the highest-probability token,
which is $v_{(1)}$ under the tie-breaking order above; temperature rescaling
with $\tau>0$ preserves this ordering. Since $\rho>0$, the nucleus set always
contains at least its first token: $v_{(1)}\in S_\rho(c)$. Applying this
argument at each prefix
$c\mathbin\Vert t_{<i}$ shows that every greedy token lies in the corresponding
nucleus set, hence the whole greedy completion is plausible.
\end{proof}

\begin{proposition}[Plausibility Check]
    Fix $f,K,\tau,\rho$ and context $c$. A text
    $t=(t_1,\ldots,t_m)$ is plausible from $c$ if and only if each token lies
    in the nucleus set determined by its preceding context:
    \begin{align}
        t \in \mathcal{P}^{\tau,\rho}_K(c)
        \iff
        |c|+m\leq K
        \text{ and }
        \forall i\leq m:\ t_i\in S_\rho(c\mathbin\Vert t_{<i}).
    \end{align}
\end{proposition}
\begin{proof}
    This is exactly Definition \ref{def:plausible-text} unpacked over the
    autoregressive generation steps.
\end{proof}

\subsection{The Weak Coding Theorem}
\label{sec:coding-theorem}
The failure of an invariance theorem means that prompting complexity cannot be
estimated by applying an arbitrary off-the-shelf compressor to the target text:
the relevant compressor is the particular model \(f\), together with its
training distribution, tokenizer, prompt template and decoding rule. What does
remain useful from the classical coding theorem is the link between description
length and probability mass. We therefore define a model-dependent analogue of
algorithmic probability by drawing a self-delimiting plausible prompt with a
length prior and asking whether deterministic decoding produces the target text.

\begin{definition}[Prompting Probability]
    Let $t \in \mathcal{P}_K$ and let $f:\mathcal{P}_K\rightarrow\mathcal{P}_K$
    be a LM. The prompting probability of the text $t$ is:
    \begin{align}
        m_f(t) = \sum_{\substack{p\in\mathcal{P}_K\\f_{\tau = 0}(p) = t}} T^{-|p|} 
    \end{align}
\end{definition}

The prompting probability counts the plausible prompts that generate a target
text, weighted by their length, making shorter prompts contribute more mass.
Unlike the universal algorithmic probability semimeasure
\cite{SOLOMONOFF19641}, $m_f(t)$ is computable in the finite LM setting,
but it depends on $f$ and is therefore not universal.

\begin{lemma}
    The prompting probability $m_f$ is computable and $\sum\limits_{t\in \mathcal{P}_K} m_f(t) \leq 1$.
\end{lemma}
\begin{proof}
  If no prompt $p$ produces a given text $t$, then the inner sum is over an empty set and so $m_f(t)=0$. Consider the collection of preimage sets $A_t =\{p \in \mathcal{P}_K : f_{\tau=0}(p)=t\},~t\in\mathcal{P}_K$.
  
  Since each prompt $p$ has exactly one output $f(p)$, the sets $\{A_t\}_{t\in\mathcal{P}_K}$ form a partition of $\mathcal{P}_K$: $\mathcal{P}_K \;=\bigcup_{t\in\mathcal{P}_K} A_t$.
  
  Because \(T^{-|p|}\ge0\), we may sum over this partition in either order:
    \begin{align}
        \sum\limits_{t\in \mathcal{P}_K} m_f(t) = \sum\limits_{t\in \mathcal{P}_K} \sum_{\substack{p\in\mathcal{P}_K\\f_{\tau = 0}(p) = t}} T^{-|p|} = \sum\limits_{p\in \mathcal{P}_K} T^{-|p|} 
    \end{align}
    By Definition \ref{def:pit-llm} of the LM, we assume that the tokenizer template unambiguously separates system prompts, input prompts, reasoning traces and output texts. Thus texts in $\mathcal{P}_K$ are self-delimiting and form a prefix-free set over an alphabet of size $T$. By Kraft's inequality \cite{kraft1949device}, $\sum\limits_{p\in \mathcal{P}_K} T^{-|p|} \leq 1$, yielding the desired bound.
    Finally, to see that each $m_f(t)$ is computable we observe that $\mathcal{P}_K$ is finite and that the function $f_{\tau=0}(p)$ is computable (Proposition \ref{prop:f-computable}). We can therefore enumerate all plausible prompts, simulate them with $f$, and add the weights of those that produce $t$.
\end{proof}

Consider the marginal model probability of a text $t = (t_1, t_2, \dots, t_k)$:
\begin{align}
    P_f(t) = \sum\limits_{p\in\mathcal{P}_K} T^{-|p|} P_f(t\mid p),
    \qquad
    P_f(t\mid p) = \prod\limits_{i=1}^{|t|}P_f(t_i\mid t_{<i}, p),
\end{align}
where $P_f(t\mid p)$ is the usual text probability under the model $f$ with
temperature $\tau = 1$. This gives a weak analogue of the Coding Theorem
\cite{li2008introduction}.

\begin{theorem}[Weak Coding Theorem]
    For any $t \in \mathcal{P}_K$ with $\Psi_f(t)<\infty$, let
    $S_t=\{p\in\mathcal{P}_K:f_{\tau=0}(p)=t\}$ and
    $Z_f(t)=\sum_{p\in S_t}T^{-(|p|-\Psi_f(t))}$. Then:
    \begin{equation}
        -\log_T P_f(t) - |t| \leq -\log_T m_f(t)
        =
        \Psi_f(t)-\log_T Z_f(t)
        \leq \Psi_f(t).
    \end{equation}
\end{theorem}
\begin{proof}
    Let $S = \{p\in\mathcal{P}_K: f_{\tau = 0}(p) = t\}$ be the set of all prompts whose temperature 0 decoding is exactly $t$. Then, $m_f(t) = \sum\limits_{p\in S} T^{-|p|}$ and
    \begin{align}
        P_f(t) = \sum\limits_{p\in\mathcal{P}_K} T^{-|p|} P_f(t|p) = \sum\limits_{p\in S} T^{-|p|} P_f(t|p) + \sum\limits_{p\notin S} T^{-|p|} P_f(t|p)
    \end{align}
    It is clear that, since the second sum is non-negative, we have that $P_f(t) \geq \sum\limits_{p\in S} T^{-|p|} P_f(t|p)$.
    Since $p \in S$, at each position $i$ greedy decoding chose $t_i$ as a highest-probability next token. Under the temperature $\tau=1$ distribution, any highest-probability token has probability at least $T^{-1}$, and therefore:
    \begin{align}
        P_f(t_i|t_{<i}, p) \geq T^{-1} \implies P_f(t|p) = \prod\limits_{i=1}^{|t|}P_f(t_i|t_{<i}, p) \geq T^{-|t|} 
    \end{align}

    Combining the two we get:
    \begin{align}
        P_f(t) \geq \sum\limits_{p\in S} T^{-|p|} P_f(t|p) \geq \sum\limits_{p\in S} T^{-|p|}T^{-|t|} = T^{-|t|}\sum\limits_{p\in S} T^{-|p|} = T^{-|t|} m_f(t) 
    \end{align}
    
    Rearranging and applying $-\log_T$ on both sides, we obtain:
    \begin{align}
        -\log_T P_f(t) - |t| \leq -\log_T m_f(t)
    \end{align}
    Which proves the first inequality. To relate $m_f(t)$ to $\Psi_f(t)$, we have by definition:
    \begin{align}
        m_f(t) = \sum_{\substack{p\in\mathcal{P}_K\\f_{\tau = 0}(p) = t}} T^{-|p|}
        =
        T^{-\Psi_f(t)}
        \sum_{p\in S}T^{-(|p|-\Psi_f(t))}
        =
        T^{-\Psi_f(t)}Z_f(t).
    \end{align}
    Taking $-\log_T$ on both sides gives:
    \begin{align}
        -\log_T m_f(t) = \Psi_f(t)-\log_T Z_f(t).
    \end{align}
    Since $S$ contains at least one shortest prompt $p^*\rightarrowtail t$, $Z_f(t)\geq 1$, and therefore $-\log_T m_f(t)\leq \Psi_f(t)$.
\end{proof}

The weak coding theorem relates the usual text probability under a pretrained
LM with the prompt mass assigned to prompts that generate the text. The term
$Z_f(t)$ measures how many plausible prompts, weighted by length relative to the
shortest one, produce the same output. When this multiplicity term is bounded,
$-\log_T m_f(t)$ and $\Psi_f(t)$ differ by at most a constant; when many prompts
lead to the same text, the prompting probability can be much larger than the
mass of a single shortest prompt. In practice, estimates of $P_f(t)$ can still
provide lower bounds on $-\log_T m_f(t)$, and therefore on the amount of prompt
mass that must be searched for when trying to find short prompts.

\end{document}